\documentclass[11pt,letterpaper]{article}
\usepackage[top=1in,bottom=1in,left=1in,right=1in]{geometry}

\usepackage{microtype}
\usepackage{amsfonts}
\usepackage{mathtools}
\usepackage{amsmath}
\usepackage{amsthm}
\DeclareMathOperator*{\argmin}{argmin}
\newtheorem{theorem}{Theorem}

\newtheorem{lemma}{Lemma}
\usepackage[usenames, dvipsnames]{color}
\usepackage{subcaption}
\usepackage{natbib} 
\usepackage{microtype}
\usepackage{graphicx}
\usepackage{booktabs}
\usepackage{xcolor}

\usepackage{color}

\begin{document}

\title{\textbf{Cost-Aware Learning for Improved Identifiability with Multiple Experiments}}
\author{
	\small Longyun Guo\\
	\small Biochemistry,Purdue University\\
	\small West Lafayette, IN 47907, USA\\
	\small guo165@purdue.edu
	\and
	\small Jean Honorio\\
	\small Computer Science, Purdue University\\
	\small West Lafayette, IN 47907, USA\\
	\small jhonorio@purdue.edu
	\and
	\small John Morgan\\
	\small Chemical Engineering, Purdue University\\
	\small West Lafayette, IN 47907, USA\\
	\small jamorgan@purdue.edu
}
\date{}
\maketitle
\date{}
\maketitle

\begin{abstract}
We analyze the sample complexity of learning from multiple experiments where the experimenter has a total budget for obtaining samples. In this problem, the learner should choose a hypothesis that performs well with respect to multiple experiments, and their related data distributions. Each collected sample is associated with a cost which depends on the particular experiments. In our setup, a learner performs $m$ experiments, while incurring a total cost $C$. We first show that learning from multiple experiments allows to improve identifiability. Additionally, by using a Rademacher complexity approach, we show that the gap between the training and generalization error is $O(C^{-\frac{1}{2}})$. We also provide some examples for linear prediction, two-layer neural networks and kernel methods.
\end{abstract}

\section{Introduction}
\label{Introduction}

Several questions in machine learning can be formulated as inferring the true hypothesis given a finite number of samples from an unknown arbitrary distribution. Different hypotheses can be evaluated with their expected losses, which are defined as the expectation of the corresponding empirical losses derived from samples. While the true hypothesis is usually associated with the minimal expected loss, in most cases the expected loss cannot be accessed due to lack of information for the underlying sampling population. Minimization of the empirical loss is thus applied to infer the true hypothesis. It is then important to evaluate the closeness between the empirical minimizer and the true hypothesis, in terms of the expected loss, especially how it is affected by the number of collected samples. 

Various techniques have been developed to study the difference of the expected loss between the empirical minimizer and the true hypothesis, by learning from one data distribution. For instance, PAC-Bayes (\cite{mcallester98}), VC dimension (\cite{vapnik71}), covering numbers (\cite{zhang02}), fat-shattering dimension (\cite{bartlett98}), as well as Rademacher and Gaussian complexities (\cite{koltchinskii00,bartlett02model,bartlett02}) to name a few.

Previous works have studied the problem of learning from one data distribution. Here, we study learning from multiple experiments, which is a more realistic problem that fits the nowadays scientific practice, in contrast with learning from just one data distribution. This is mainly due to the issue of identifiability. That is, multiple hypotheses can potentially have the same expected loss, for the same experiment, which makes it difficult to discern which hypothesis to prefer. On the other hand, one hypothesis can stand out from the rest by performing multiple experiments with different data distributions. Here we assume a total cost budget $C$ and $m$ experiments to perform, and analyze the number of samples needed for each of the $m$ experiments, so that the gap between the training and generalization error is minimized.

In this paper, we develop a general framework for learning from multiple experiments. We first show that multiple experiments improve identifiability, by reducing the set of optimal hypotheses. Additionally, we study the sample complexity of the problem. With the assumption that the Rademacher complexity of each experiment is on the order of $O(n^{-\frac{1}{2}})$, we show that the uniform convergence is at a rate of $O(C^{-\frac{1}{2}})$, where $C$ is the total cost budget to be distributed across the $m$ different experiments. We also provide some examples in linear prediction, two-layer neural networks, and kernel methods.

\section{Preliminaries}
We assume that there is a true hypothesis $h^*\in\mathcal{F}$, where $\mathcal{F}$ is the hypothesis set. Additionally, we assume that there is a finite experiment set $\mathfrak{D} = \{ \mathcal{D}_1, \mathcal{D}_2, \dots, \mathcal{D}_m \}$. Each element $\mathcal{D}_j\in\mathfrak{D}$ is a data distribution where samples are drawn from.

We further assume that $\mathfrak{D}$ comes with a per-sample cost set $\mathcal{C}=\{ c_1,c_2,\dots, c_m \}$. This assumption comes from the fact that in practice, different experiments require different amount of resources. All these investments are summarized as experimental costs. Furthermore, it is reasonable to assume that the total cost for one experiment is proportional to the number of samples used in the experiment.

One experiment is said to be performed if some samples are collected from the corresponding data distribution. In that case, a dataset $S_j=\{z_{j,1},\dots,z_{j,n_j}\}$ with $n_j$ samples is obtained by drawing from $\mathcal{D}_j\in\mathfrak{D}$. If one is constrained with a total cost $C$ to perform $m$ experiments, then the numbers of samples for $m$ experiment are constrained in the following fashion:

\begin{equation}
\label{cost_constraint}
\sum_{j=1}^{m} c_j n_j \le C
\end{equation}

Given the expected loss $\mathbb{E}_{z_{j}\sim\mathcal{D}_j}[h(z_{j})]$ for $\forall h \in \mathcal{F}$ and $\forall\mathcal{D}_j\in\mathfrak{D}$ within the range of $[0, 1]$, a combined expected loss over $m$ experiments can then be defined:

\begin{equation}
\label{combined_expected_loss}
\mathbb{E}_{\mathcal{D}_1^m}[h] = \frac{1}{m} \sum_{j=1}^{m} \mathbb{E}_{z_j\sim\mathcal{D}_j}[h(z_j)]
\end{equation}

The true hypothesis $h^*$ is assumed to satisfy:

\begin{equation}
\label{h_start_constraint}
(\forall \ j)\ h^* \in \mathcal{H}_j^* \equiv \argmin_{%
	\substack{%
		h\in\mathcal{F}
	}
}
\mathbb{E}_{z_{j}\sim\mathcal{D}_j}[h(z_{j})]
\end{equation}

Note that for any given data distribution $\mathcal{D}_j$ there could exist other hypotheses with the same expected loss $\mathbb{E}_{z_j\sim\mathcal{D}_j}[h(z_j)]$, making it impossible to discern between them by only learning from one data distribution.

To evaluate hypotheses with a finite number of samples, a combined empirical loss over $m$ experiments is defined as:

\begin{equation}
\hat{\mathbb{E}}_{S_1^m}[h]=\frac{1}{m} \sum_{j=1}^{m}\frac{1}{n_j}\sum_{i=1}^{n_j}h(z_{j,i})
\end{equation}

Thus, the empirical hypothesis learned from $m$
experiments satisfies the following condition:

\begin{equation}
\hat{h} \in \argmin_{%
	\substack{%
		h\in\mathcal{F}
	}
}
\hat{\mathbb{E}}_{S_1^m}[h]
\end{equation}

To measure the difference between any two hypotheses, a divergence function $d_m: \mathcal{F}\times\mathcal{F}\to[-1,1]$ is defined as:

\begin{equation}
(\forall h,h'\in\mathcal{F})\ d_m(h,h')=\mathbb{E}_{\mathcal{D}_{1}^{m}}[h]-
\mathbb{E}_{\mathcal{D}_{1}^{m}}[h']
\end{equation}

The empirical hypothesis $\hat{h}$ is said to recover the true hypothesis $h^*$ if we can show that $d_m(\hat{h},h^*)\to0$ as $C\to\infty$. In this paper, we identify the dependence of $d_m(\hat{h},h^*)$ with respect to the number of experiments $m$, as well as the total cost $C$.

\section{Results}

First, we show that learning from multiple experiments can improve hypothesis identifiability when compared to learning from single experiments, which justifies our learning problem.

\begin{theorem}
	\label{multi_exp_theorem}
	Let 
	\begin{equation*}
	\mathcal{H}^* \equiv \argmin_{%
		\substack{%
			h\in\mathcal{F}
		}
	}
	\mathbb{E}_{\mathcal{D}_1^m}[h]
	\end{equation*}
	If $(\forall \ j)\ h^* \in \mathcal{H}_j^*$, then the following holds:\\
	\begin{equation*}
	h^* \in \mathcal{H}^* = \mathcal{H}_1^* \cap \mathcal{H}_2^* \cap \dots \cap \mathcal{H}_m^*.
	\end{equation*}
\end{theorem}

(Detailed proofs can be found in Appendix~\ref{sec:multi_exp_theorem_proof}.)

Note that $\mathcal{H}^*$ is the intersection of $m$ sets. Thus as $m$ increases, the size of $\mathcal{H}^*$ decreases. The fact that the size of $\mathcal{H}^*$ decreases improves the identifiability of the true hypothesis.

In what follows, we concentrate on the sample complexity of learning from multiple experiments. We provide several theorems in order to upper-bound $d_m(\hat{h},h^*)$ with respect to the number of experiments $m$, as well as the total cost $C$. In order to estimate $d_m(\hat{h},h^*)$, we make use of the empirical Rademacher complexity of the hypothesis class $\mathcal{F}$ with respect to the datasets $S_j$ of $n_j$ samples, defined as:

\begin{equation}
\hat{\Re}_{S_j}(\mathcal{F})=\mathbb{E}_{\sigma}\bigg[\sup\limits_{h \in \mathcal{F}} \bigg(\frac{1}{n_j}\sum_{i=1}^{n_j}\sigma_i h(z_{j,i})\bigg)\bigg]
\end{equation}

\noindent where $\sigma=\{\sigma_1, \dots \sigma_{n_j}\}$ are $n_j$ independent Rademacher random variables, which are uniform $\{\pm1\}$-valued. The Rademacher complexity of the hypothesis class $\mathcal{F}$ for $n_j$ samples is defined as:

\begin{equation}
\Re_{n_j}(\mathcal{F})=\mathbb{E}_{S_j\sim\mathcal{D}_j^{n_j}}[\hat{\Re}_{S_j}(\mathcal{F})]
\end{equation}

In addition, two functions describing the maximal difference between $\mathbb{E}_{\mathcal{D}_{1}^{m}}[h]$ and $\hat{\mathbb{E}}_{S_1^m}[h]$ over $\mathcal{F}$ are defined:

\begin{equation}
\varphi(S)=\sup\limits_{h \in \mathcal{F}}\bigg(\mathbb{E}_{\mathcal{D}_{1}^{m}}[h]-\hat{\mathbb{E}}_{S_1^m}[h]\bigg)
\end{equation}

\begin{equation}
\varphi'(S)=\sup\limits_{h \in \mathcal{F}}\bigg(\hat{\mathbb{E}}_{S_1^m}[h]-\mathbb{E}_{\mathcal{D}_{1}^{m}}[h]\bigg)
\end{equation}

The following two lemmas are introduced to help bounding $d_m(\hat{h},h^*)$.

\begin{lemma}
	\label{lemma1}
	The following holds:
	\begin{align*}
	\mathbb{P}\bigg[&f(z_{1,1} \dots z_{j,i} \dots z_{m,n_m})-\mathbb{E}[f(z_{1,1} \dots z_{j,i} \dots z_{m,n_m})] \ge\epsilon\bigg]\le e^{\frac{-2 m^2 \epsilon^2}{\sum_{j=1}^{m} \frac{1}{n_j}}}
	\end{align*}
	for $f(S)=\frac{1}{m} \sum_{j=1}^m \hat{\Re}_{S_j}(\mathcal{F})$, $f(S)=\varphi(S)$ or $f(S)=\varphi'(S)$.
\end{lemma}

\begin{proof}
	All the above definitions of $f(S)$ satisfy the following condition, as both expected losses and empirical losses for all data distributions are bounded within $[0, 1]$:
	
	$(\forall i, j, \forall z_{j,i}, \tilde{z}_{j,i} \sim \mathcal{D}_j)|f(z_{1,1} \dots z_{j,i} \dots z_{m,n_m})-f(z_{1,1} \dots \tilde{z}_{j,i} \dots z_{m,n_m})|\le\frac{1}{m n_j}$
	
	According to McDiarmid's inequality \cite{mcdiarmid89}, we prove our claim.
\end{proof}

\begin{lemma}
	\label{lemma2}
	The following holds:
	\begin{equation*}
	\mathbb{E}_{S_1^m}[\varphi(S)]\le \frac{2}{m} \sum_{j=1}^{m} \Re_{n_j}(\mathcal{F})
	\end{equation*}
	and 
	\begin{equation*} 
	\mathbb{E}_{S_1^m}[\varphi'(S)]\le \frac{2}{m} \sum_{j=1}^{m} \Re_{n_j}(\mathcal{F}).
	\end{equation*}
\end{lemma}

(Detailed proofs can be found in Appendix~\ref{sec:lemma2_proof}.)

Given the above lemmas, we provide our bound for $d_m(\hat{h},h^*)$ with respect to the Rademacher complexity in the following theorem, which follows from Lemma~\ref{lemma1} and Lemma~\ref{lemma2}, as well as union bound arguments.

\begin{theorem}
	\label{th_expected_R}
	The divergence over $m$ experiments is bounded as follows:
	
	$d_m(\hat{h},h^*)\le \frac{4}{m} \sum_{j=1}^{m} \Re_{n_j}(\mathcal{F})
	+\frac{1}{m} \sqrt{2 \log{\frac{2}{\delta}} \sum_{j=1}^{m}\frac{1}{n_j}}$ with a probability at least $1-\delta\ (\delta\in(0,1))$.
\end{theorem}
(Detailed proofs can be found in Appendix~\ref{sec:th1_proof}.)

Similarly, we provide our bound for $d_m(\hat{h},h^*)$ with respect to the empirical Rademacher complexity in the following theorem, which follows from Lemma~\ref{lemma1} and Lemma~\ref{lemma2}, as well as union bound arguments.

\begin{theorem}
	\label{th_emp_R}
	The divergence over $m$ experiments is bounded as follows:
	
	$d_m(\hat{h},h^*)\le \frac{4}{m} \sum_{j=1}^{m} \hat{\Re}_{S_j}(\mathcal{F})
	+\frac{1}{m} \sqrt{18 \log{\frac{3}{\delta}} \sum_{j=1}^{m}\frac{1}{n_j}}$ with a probability at least $1-\delta\ (\delta\in(0,1))$.
\end{theorem}

(Detailed proofs can be found in Appendix~\ref{sec:th2_proof}.)

The bounds on $d_m(\hat{h},h^*)$ are dependent on the number of samples for each experiment as shown in Theorem~\ref{th_expected_R} and Theorem~\ref{th_emp_R}. We can further adjust the bounds by identifying the optimal strategy to determine the number of samples for each experiment, so that the bounds are minimal under the constraint in \eqref{cost_constraint}.

Here we assume that the Rademacher complexity is on the order of $O(n^{-\frac{1}{2}})$ where $n$ is the number of collected samples in one experiment. In fact, there is a wide range of examples satisfying this requirement, such as the empirical Rademacher complexity of linear predictors with different constraints (\cite{kakade08}); and the Rademacher complexity of two-layer neural networks as well as kernel methods (\cite{bartlett02}).

Further adjustments on the bounds can be made with the following two theorems:

\begin{theorem}
	\label{th_expected_R_2}
	Given $\Re_{n_j}(\mathcal{F})\le\frac{a_j}{\sqrt{n_j}}$ where $a_j\ge 0$ and \eqref{cost_constraint}, we have that
	
	\begin{align*}
	n_j = \frac{C \sqrt{16 a_j^2 + 2 \log{\frac{2}{\delta}}}}{\sqrt{c_j} \sum_{k=1}^m \sqrt{16 a_k^2 c_k + 2 c_k \log{\frac{2}{\delta}}}}
	\end{align*}
	
	and  
	
	\begin{align*}
	d_m(\hat{h},h^*)&\le \frac{\sqrt{\sum_{j=1}^m \sqrt{16 a_j^2 c_j + 2 c_j \log{\frac{2}{\delta}}}}}{m\sqrt{C}}\bigg[ \sum_{j=1}^m \frac{4 a_j \sqrt[4]{c_j}}{\sqrt[4]{16 a_j^2 + 2 \log{\frac{2}{\delta}}}} + \sqrt{2 \log{\frac{2}{\delta}} \sum_{j=1}^m \frac{\sqrt{c_j}}{\sqrt{16 a_j^2 + 2 \log{\frac{2}{\delta}}}}} \bigg] \\
	&\le \frac{\sqrt{m+1}}{m\sqrt{C}}\sum_{j=1}^m \sqrt{16 a_j^2 c_j +2 c_j \log{\frac{2}{\delta}}}
	\end{align*}
	
	with a probability at least $1-\delta\ (\delta\in(0,1))$.
	
\end{theorem}

(Detailed proofs can be found in Appendix~\ref{sec:th3_proof}.)

\begin{theorem}
	\label{th_emp_R_2}
	Given $\hat{\Re}_{S_j}(\mathcal{F})\le\frac{a_j}{\sqrt{n_j}}$ where $a_j\ge 0$ and \eqref{cost_constraint}, we have that 
	
	\begin{align*}
	n_j = \frac{C \sqrt{16 a_j^2 + 18 \log{\frac{3}{\delta}}}}{\sqrt{c_j} \sum_{k=1}^m \sqrt{16 a_k^2 c_k + 18 c_k \log{\frac{3}{\delta}}}}
	\end{align*}
	
	and 
	
	\begin{align*}
	d_m(\hat{h},h^*) &\le \frac{\sqrt{\sum_{j=1}^m \sqrt{16 a_j^2 c_j + 18 c_j \log{\frac{3}{\delta}}}}}{m\sqrt{C}}\bigg[\sum_{j=1}^m \frac{4 a_j \sqrt[4]{c_j}}{\sqrt[4]{16 a_j^2 + 18 \log{\frac{3}{\delta}}}} + \sqrt{18 \log{\frac{3}{\delta}} \sum_{j=1}^m \frac{\sqrt{c_j}}{\sqrt{16 a_j^2 + 18 \log{\frac{3}{\delta}}}}} \bigg] \\
	&\le \frac{\sqrt{m+1}}{m\sqrt{C}}\sum_{j=1}^m \sqrt{16 a_j^2 c_j+18 c_j \log{\frac{3}{\delta}}}
	\end{align*}
	
	with a probability at least $1-\delta\ (\delta\in(0,1))$.
	
\end{theorem}

\begin{proof}
	We proceed with the proof as in Theorem~\ref{th_expected_R_2}.
\end{proof}

From Theorem~\ref{th_expected_R_2} and Theorem~\ref{th_emp_R_2} we can observe that $d_m(\hat{h},h^*)$ converges to $0$ at a rate of $O(C^{-\frac{1}{2}})$. On the other hand, Theorem~\ref{th_expected_R_2} and Theorem~\ref{th_emp_R_2} provide the way to determine the number of samples needed for each of the $m$ experiments. By assuming $a_j^2 \gg \log{\frac{1}{\delta}}$, the number of samples for one experiment should be proportional to a constant factor of its Rademacher complexity, meanwhile inversely proportional to the square root of its per-sample cost $c_j$.
\begin{table*}[t]
	\caption{Rates of Learning from Multiple Experiments for Different Problems}
	\label{table1}
	\centering
	\begin{tabular}{lll}
		\toprule
		PREDICTOR     & UPPER BOUND ON $d_m(\hat{h},h^*)$ \\
		\midrule
		Linear Predictors ($L_2/L_2$ norms)
		& $\frac{\sqrt{m+1}}{m\sqrt{C}}\sum_{j=1}^m \sqrt{16 X_{2,j}^2 W_2^2 c_j+18 c_j \log{\frac{3}{\delta}}}$\\
		Linear Predictors ($L_\infty/L_1$ norms)
		& $\frac{\sqrt{m+1}}{m\sqrt{C}}\sum_{j=1}^m \sqrt{32 X_{\infty,j}^2 W_1^2 c_j \log{l} +18 c_j \log{\frac{3}{\delta}}}$\\
		Two-Layer Neural Networks
		& $\frac{\sqrt{m+1}}{m\sqrt{C}}\sum_{j=1}^m \sqrt{16 B^2 X_{\infty,j}^2 c_j \log{l} +2 c_j \log{\frac{2}{\delta}}}$\\
		Kernel Methods
		& $\frac{\sqrt{m+1}}{m\sqrt{C}} \sum_{j=1}^{m} \sqrt{64 B_j^2 \mathbb{E}_{x_j}[k(x_j, x_j)] c_j +2 c_j \log{\frac{2}{\delta}}}$\\
		\bottomrule
	\end{tabular}
\end{table*}

\section{Examples}

While the upper bounds provided in Theorem~\ref{th_expected_R_2} and Theorem~\ref{th_emp_R_2} have a rate of $O(C^{-\frac{1}{2}})$, they are also dependent on $\frac{\sqrt{m+1}}{m}\sum_{j=1}^m \sqrt{16 a_j^2 c_j+2 c_j \log{\frac{2}{\delta}}}$. Here we present some specific examples with Theorem~\ref{th_expected_R_2} to intuitively understand the behavior of this term. Additionally, we provide some examples for linear prediction, two-layer neural networks and kernel methods.

\subsection{Experiments with Rademacher Complexity of Large Constant Factors}
\label{sec:increasing_complexity}

Assume $\Re_{n_j}(\mathcal{F})\le\frac{a_j}{\sqrt{n_j}}$ where $a_j \gg \frac{1}{4} \sqrt{2 \log{\frac{2}{\delta}}}$. Then from Theorem~\ref{th_expected_R_2} we get:

\begin{align}
\label{example_bound}
d_m(\hat{h},h^*)&\le \frac{\sqrt{\sum_{j=1}^m \sqrt{16 a_j^2 c_j + 2 c_j \log{\frac{2}{\delta}}}}}{m\sqrt{C}}\bigg[ \sum_{j=1}^m \frac{4 a_j \sqrt[4]{c_j}}{\sqrt[4]{16 a_j^2 + 2 \log{\frac{2}{\delta}}}} + \sqrt{2 \log{\frac{2}{\delta}} \sum_{j=1}^m \frac{\sqrt{c_j}}{\sqrt{16 a_j^2 + 2 \log{\frac{2}{\delta}}}}} \bigg]\nonumber\\
&\approx \frac{\sqrt{\sum_{j=1}^m \sqrt{16 a_j^2 c_j}}}{m\sqrt{C}} \sum_{j=1}^m \frac{4 a_j \sqrt[4]{c_j}}{\sqrt[4]{16 a_j^2}}\nonumber\\
&= \frac{\sqrt{\sum_{j=1}^m 4 a_j \sqrt{c_j}} \sum_{j=1}^m \sqrt{4 a_j \sqrt{c_j}}}{m\sqrt{C}}\nonumber\\
&\le \frac{4 \big( \sum_{j=1}^m \sqrt{a_j \sqrt{c_j}} \big)^2}{m\sqrt{C}}
\end{align}

To give one specific example, here we assume that $a_j= A j^2$, and $c_j= K e^{-sj} (s>0) \ j=1,2,\dots,m$, where $A, K>0$ are absolute constants. Then from \eqref{example_bound} we get:

\begin{align*}
d_m(\hat{h},h^*)&\le \frac{4 \big( \sum_{j=1}^m \sqrt{a_j \sqrt{c_j}} \big)^2}{m\sqrt{C}}\\
&\le \frac{4 \big( \sum_{j=1}^\infty \sqrt{a_j \sqrt{c_j}} \big)^2}{m\sqrt{C}}\\
&= \frac{4 A \sqrt{K} e^{\frac{s}{2}} }{ (e^{\frac{s}{4}}-1)^4 m\sqrt{C}}
\end{align*}

Note that our bound is clearly upper bounded by a value on the order of $O(m^{-1} C^{-\frac{1}{2}})$.

\subsection{Some Learning Problems}
\label{sec:some_examples}

Here we present some examples of learning from multiple experiments, for problems with Rademacher complexity upper-bounded by $O(n^{-\frac{1}{2}})$, where $n$ is the number of collected samples in one experiment. We summarize our learning bounds in Table~\ref{table1}.

Define $\mathcal{G}=\big\{g\big|(\forall j)\ g:\mathcal{X}_j\to\mathcal{Y}_j\big\}$, so that we can define accordingly $\mathcal{F}=\big\{h(z_j)=L(y_j,g(x_j))\big|z_j=(x_j,y_j), g\in\mathcal{G}\big\}$ where $(\forall j)\ x_j\in\mathcal{X}_j,y_j\in\mathcal{Y}_j$. Assume $(\forall j)\ L: \mathcal{Y}_j\times\mathcal{Y}_j\to[0, 1]$ to be a $1$-Lipschitz function. For regression, we assume $L(y_j,y_j')=min(1,\frac{(y_j-y_j')^2}{2})$ where $y_j\in\mathbb{R}$. For classification, we assume $L(y_j,y_j')=max(0,\frac{1-y_jy_j'}{2})$ where $y_j\in\{-1,1\}$.

Note that by Ledoux-Talagrand contraction (\cite{ledoux2013probability}), the following holds:

\begin{equation}
\label{LT_emp}
\hat{\Re}_{S_j}(\mathcal{F})\le\hat{\Re}_{S_j}(\mathcal{G})
\end{equation}
and
\begin{equation}
\label{LT_exp}
\Re_{n_j}(\mathcal{F})\le\Re_{n_j}(\mathcal{G})
\end{equation}

\paragraph{Linear Predictors ($L_2/L_2$ norms).}
\label{sec:linear}

Assume $\mathcal{G}$ is a set of linear predictors, let $(\forall j)\ \Vert x_j \Vert_2 \le X_{2,j}$, $\mathcal{G}=\big\{w^T x \big| \Vert w \Vert_2 \le W_2\big\}$. By Theorem 1 in \cite{kakade08} and from \eqref{LT_emp} we have:

\begin{equation}
\hat{\Re}_{S_j}(\mathcal{F})\le\frac{X_{2,j} W_2}{\sqrt{n_j}}
\end{equation}

\paragraph{Linear Predictors ($L_\infty/L_1$ norms).}

Assume $\mathcal{G}$ is a set of linear predictors, let $(\forall j)\ x_j \in \mathbb{R}^l$, $\Vert x_j \Vert_{\infty} \le X_{\infty,j}$, $\mathcal{G}=\big\{w^T x \big| \Vert w \Vert_1 \le W_1\big\}$. By Theorem 1 in \cite{kakade08} and from \eqref{LT_emp} we have:

\begin{equation}
\hat{\Re}_{S_j}(\mathcal{F})\le\frac{X_{\infty,j} W_1 \sqrt{2 \log{l}}}{\sqrt{n_j}}
\end{equation}

\paragraph{Two-Layer Neural Networks.}

Assume $\mathcal{G}= \big\{\sum_i w_i t(v_i^T x_j) \big| \Vert w \Vert_1 \le 1, (\forall i)\ \Vert v_i \Vert_1 \le B\big\}$ with a $1$-Lipschitz function $t: \mathbb{R} \to [-1, 1]$ satisfying $t(0)=0$. Let $(\forall i,j)\ x_j \in \mathbb{R}^l$, $v_i \in \mathbb{R}^l$ with the constraints of $\Vert x_j \Vert_{\infty} \le X_{\infty,j}$. By Theorem 18 in \cite{bartlett02}, Lemma 4 in \cite{bartlett02} and \eqref{LT_exp} we have (See Appendix~\ref{sec:neural_networks} for a detailed proof)

\begin{equation}
\Re_{n_j}(\mathcal{F})\le\frac{B X_{\infty,j} \sqrt{\log{l}} }{\sqrt{n_j}}
\end{equation}

with $B>0$ being an absolute constant.

\paragraph{Kernel Methods.}

Assume $\mathcal{G}=\big\{\sum_{i=1}^{n_j} \alpha_i k(x_{j,i}, x_j)
\big| \sum_{i,k} \alpha_i \alpha_k k(x_{j,i}, x_{j,k}) \le B_j^2\big\}$ is a kernel expansion with $(\forall j)\ x_j \in \mathcal{X}_j$ and a kernel $(\forall j)\ k: \mathcal{X}_j \times \mathcal{X}_j \to \mathbb{R}$. By Lemma 22 in \cite{bartlett02} and \eqref{LT_exp} we have (See Appendix~\ref{sec:kernel_methods} for a detailed proof):

\begin{equation}
\Re_{n_j}(\mathcal{F})\le\frac{2 B_j \sqrt{\mathbb{E}_{x_j}[k(x_j, x_j)]}}{\sqrt{n_j}}
\end{equation}

\begin{figure*}[t]
	\centering
	\begin{subfigure}{0.24\textwidth}
		\centering
		\includegraphics[width=\textwidth]{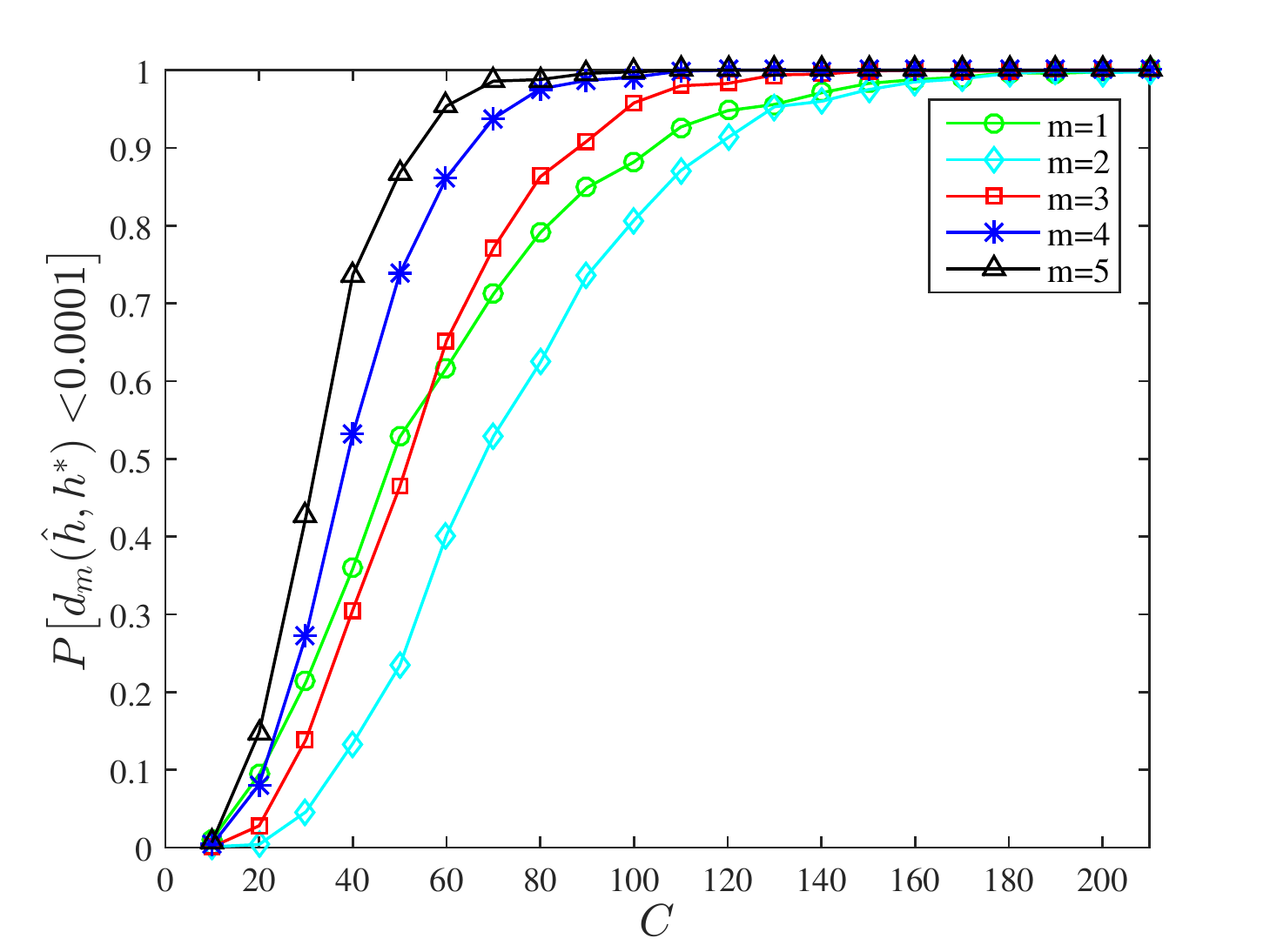}
		\caption{}
		\label{decrease_X_dm}
	\end{subfigure}
	\begin{subfigure}{0.24\textwidth}
		\centering
		\includegraphics[width=\columnwidth]{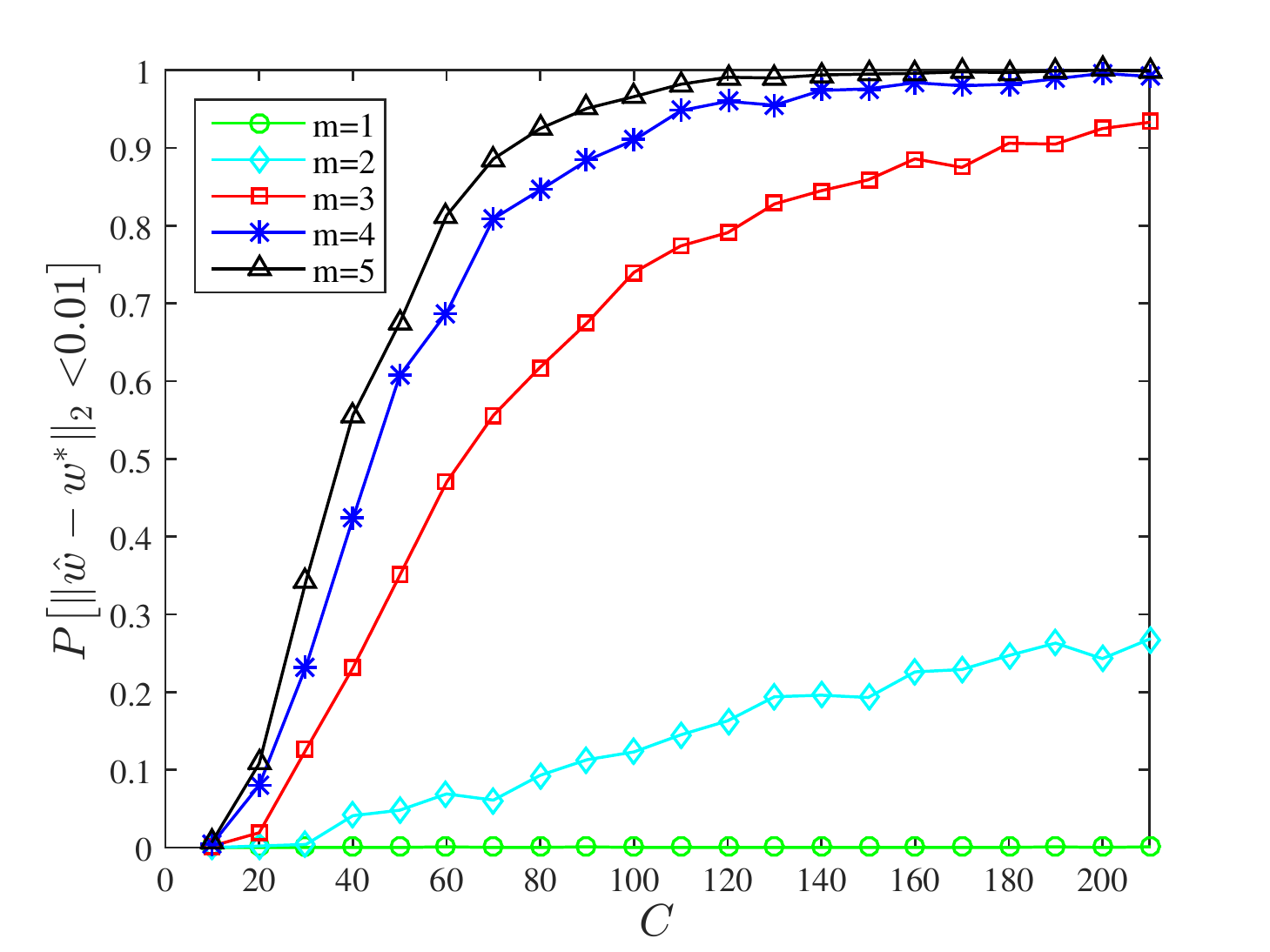}
		\caption{}
		\label{decrease_X_w_norm}
	\end{subfigure}
	\begin{subfigure}{0.24\textwidth}
		\centering
		\includegraphics[width=\textwidth]{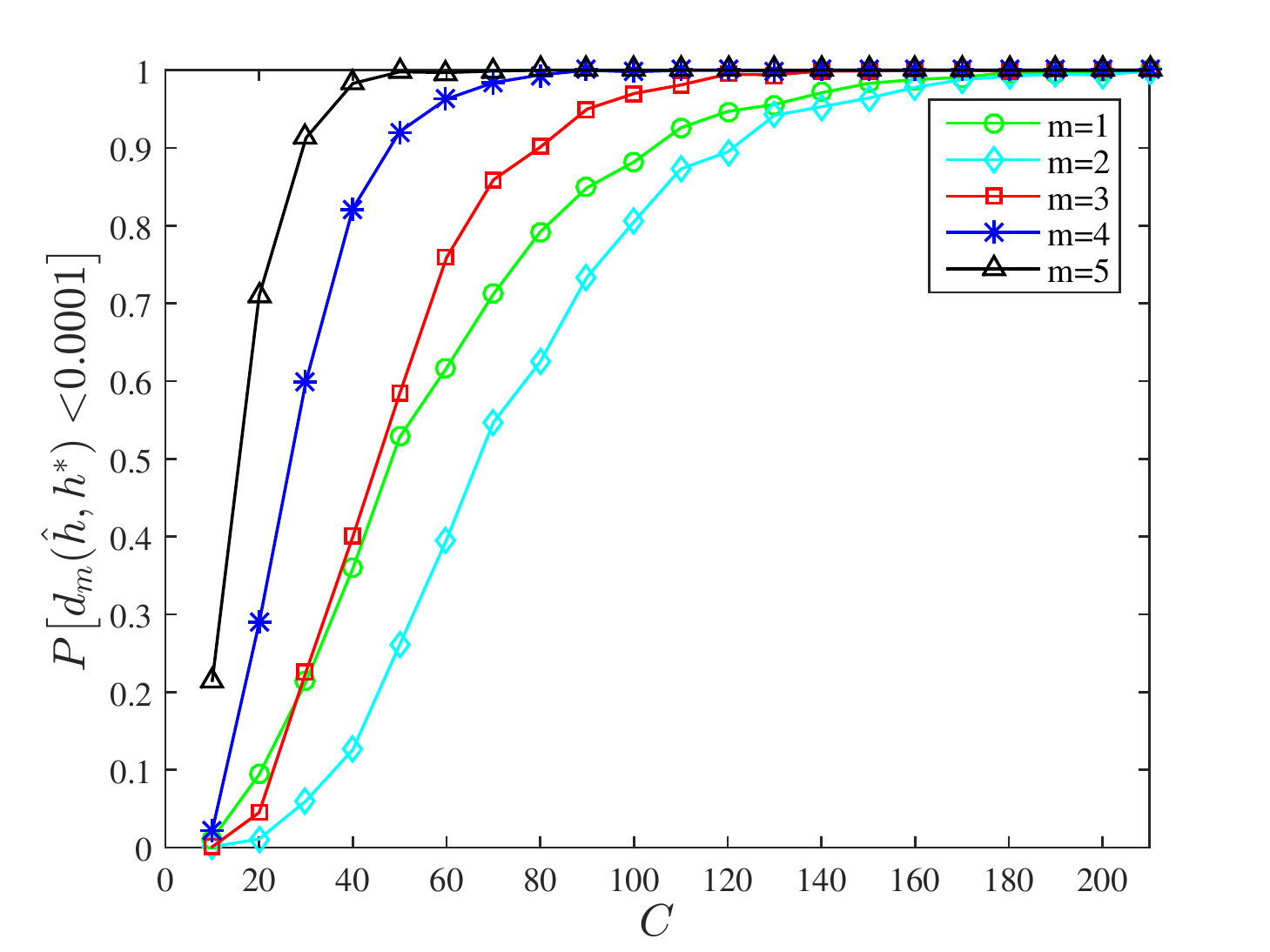}
		\caption{}
		\label{increase_X_dm}
	\end{subfigure}
	\begin{subfigure}{0.24\textwidth}
		\centering
		\includegraphics[width=\columnwidth]{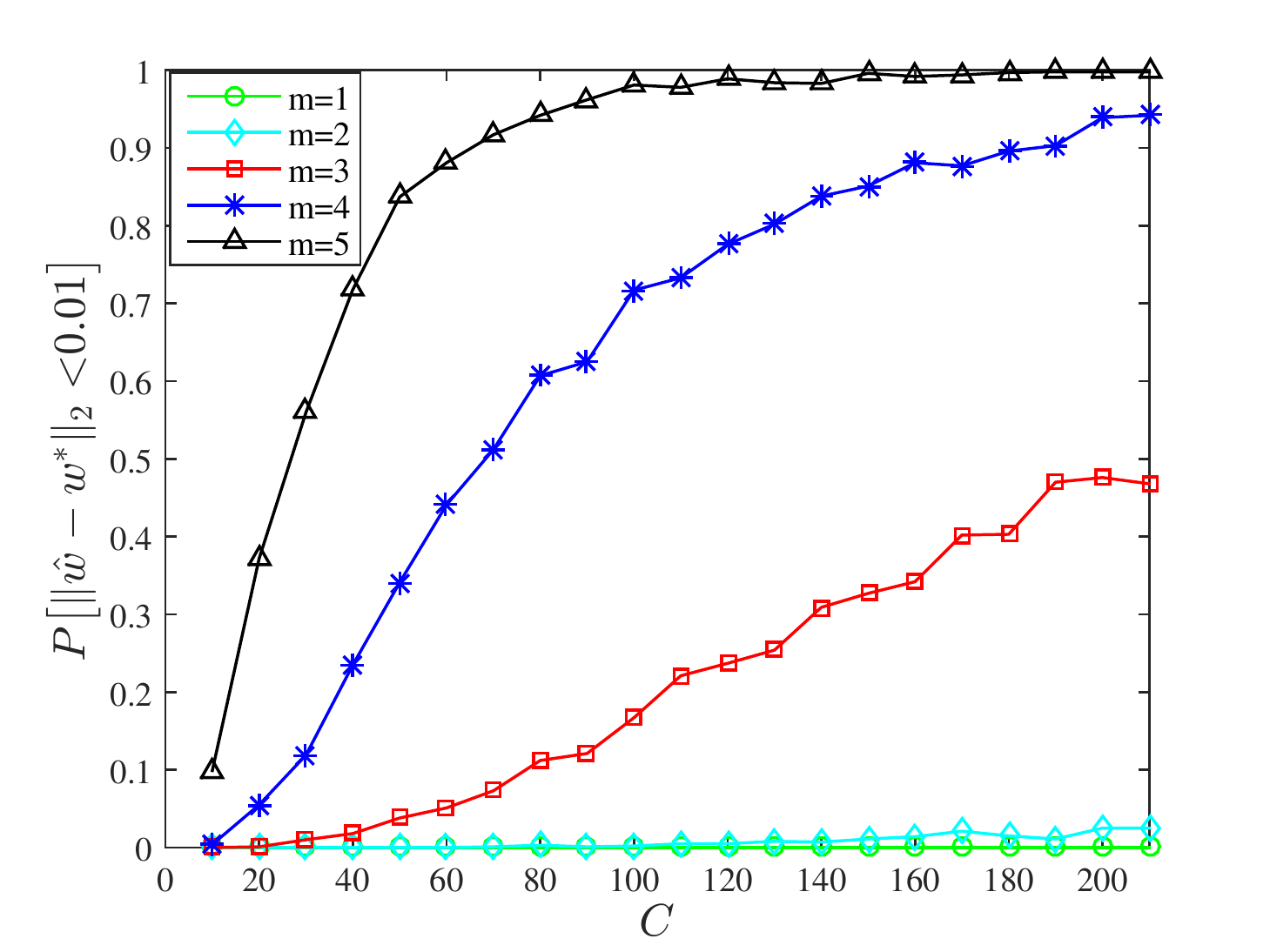}
		\caption{}
		\label{increase_X_w_norm}
	\end{subfigure}
	\caption{(a,c) Probability of success for $d_m (\hat{h}, h^*)<0.0001$ versus the total cost $C$. (b,d) Probability of success for $\Vert \hat{w} - w^* \Vert_2 <0.01$ versus the total cost $C$. The additional parameters were set to $W_2=10, s=1, l=10, \epsilon_b=0.1$. $X_{2,j} = l^{\frac{(5-j)}{8}}, j = 1,\dots, 5$ for (a,b) and $X_{2,j} = l^{\frac{(j-1)}{8}}, j = 1,\dots, 5$ for (c,d). Each point is the average result of 1000 repetitions.}
	\label{Numerical_Examples}
\end{figure*}

\section{Experiments}

In this section we present numerical validation of the proposed learning framework with synthetic datasets. The synthetic dataset for the $j$th experiment was generated from a hidden variable $\zeta_{j}$, with each hidden sample $\zeta_{j,i}=(\zeta_{j,i}^{(1)},\dots,\zeta_{j,i}^{(l/2)})$ being a $l/2$-dimensional vector, satisfying $\zeta_{j,i}^{(k)} \sim unif(-\frac{2 X_{2,j}}{l^{\frac{3}{2}}},\frac{2 X_{2,j}}{l^{\frac{3}{2}}})$. Each observed sample $x_{j,i}$ was then generated with the following equation:
\begin{equation}
x_{j,i} = A_j \zeta_{j,i}
\end{equation}

\noindent where $A_j \in \{+1,-1\}^{l \times l/2}$ projects the lower-dimensional hidden variables for the $j$-th experiment with $\pm 1$ randomly assigned to each entry of $A_j$, so that each sample $x_{j,i}=(x_{j,i}^{(1)},\dots,x_{j,i}^{(l)})$ is a $l$-dimensional vector satisfying $\Vert x_{j,i} \Vert_2 \le X_{2,j}$. By assuming $\mathcal{F}$ to be a set of linear predictors, each hypothesis $h \in \mathcal{F}$ is associated with a weight vector $w\ (\Vert w \Vert_2 \le W_2)$. The true hypothesis $w^*$ was generated randomly through $w^{*(k)} \sim unif(-\frac{W_2}{\sqrt{l}}, \frac{W_2}{\sqrt{l}})$. In this way, $x_j$ and $w$ are thus following the constraints of $L_2/L_2$ norms as in Section~\ref{sec:linear}. The output $y_j$ was generated by $y_{j,i} = w^{*T} x_{j,i} + \epsilon$ where $\epsilon \sim unif(-\epsilon_b,\epsilon_b)$. Note that since all $x_{j}$ are generated from lower-dimensional hidden variables $\zeta_{j}$, then $w^*$ is not identifiable by learning from one experiment only.

We also define the per-sample cost to be $c_j = (e^s-1)e^{-sj} (s>0) \ j=1,2,\dots,m$, as discussed in Section~\ref{sec:increasing_complexity}.

For each simulation, given the total cost $C$, we determined the number of samples for each experiment $n_j$ according to:

\begin{equation}
\label{n_j_split2}
n_j=C \frac{X_{2,j}}{\sqrt{c_j} \sum_{k=1}^{m} X_{2,k} \sqrt{c_k}}
\end{equation}

Here we make $X_{2,j} W_2$ to be sufficiently large, so that the number of samples determined by \eqref{n_j_split2} is close to the one prescribed by Theorem~\ref{th_expected_R_2}, thus, making it possible to disregard the term $\delta$.

After generating samples, $\hat{w}$ is identified by solving the following optimization problem:

\begin{equation*}
\hat{w}=\argmin_{%
	\substack{%
		\Vert w \Vert_2 \le W_2
	}
}
\frac{1}{m} \sum_{j=1}^m \frac{1}{2 n_j} \sum_{i=1}^{n_j} (w^T x_{j,i} - y_{j,i})^2
\end{equation*}
In order to evaluate our framework, we derived $d_m (\hat{h}, h^*)$ analytically. More precisely, $d_m (\hat{h}, h^*) = \frac{2}{3ml^3} \sum_{j=1}^m X_{2,j}^2 \Vert (\hat{w}-w^*)^T A_j \Vert_2^2$. (Details can be found in Appendix~\ref{sec:experiment_details}).

To empirically evaluate the outcome of having multiple experiments with the proposed framework, a sequence of $X_{2,j}$ was designed such that $X_{2,j}$ was either in a decreasing order (Figure~\ref{decrease_X_dm} and \ref{decrease_X_w_norm}) or an increasing order (Figure~\ref{increase_X_dm} and \ref{increase_X_w_norm}). For each simulation with $m$ experiments, $\hat{w}$ was learned from $m$ datasets with constraints of $X_{2,1}$ to $X_{2,m}$. Figure~\ref{decrease_X_dm} and \ref{increase_X_dm} suggest that $d_m (\hat{h}, h^*)$ can be reduced with higher total cost $C$ regardless of the number of experiments, and more experiments in general leads to a faster reduction. On the other hand, as it can be seen from Figure~\ref{decrease_X_w_norm} and \ref{increase_X_w_norm}, while recovery of $w^*$ can be guaranteed almost equally well for learning from four or more experiments when the total cost $C$ is sufficiently large ($C>200$ in this case), in general, learning with more experiments clearly shows a better performance for a wide range of values of $C$. Moreover, $w^*$ cannot be recovered correctly when less than three experiments were performed. This verified the benefits of having more experiments to improve the hypothesis identifiability, as proved in Theorem~\ref{multi_exp_theorem}.

\section{Concluding Remarks}
A direct extension of current work is to derive the upper bound of $d_m(\hat{h},h^*)$ given different forms of Rademacher complexities, especially the ones on the order of $O(n^k)$ with $k>-\frac{1}{2}$. Another interesting direction is to propose some underlying mechanisms to connect the Rademacher complexity with the per-sample cost for each experiment, so that the whole bound can be more tightly associated with the experiment design given a total cost budget.

\bibliographystyle{plainnat}
\bibliography{Learning_from_Multiple_Experiments}

\clearpage
\appendix
\section{Detailed Proofs}
\label{sec:detailed_proofs}

\subsection{Proof of Theorem 1}
\label{sec:multi_exp_theorem_proof}

\begin{proof}
	Since $(\forall j) h^* \in \mathcal{H}_j^*$, then for any $h' \ne h^*$, from \eqref{combined_expected_loss} and \eqref{h_start_constraint} we have:
	
	\begin{align*}
	\mathbb{E}_{\mathcal{D}_1^m}[h'] &= \frac{1}{m} \sum_{j=1}^{m} \mathbb{E}_{z_j\sim\mathcal{D}_j}[h'(z_j)]\\
	&\ge \frac{1}{m} \sum_{j=1}^{m} \mathbb{E}_{z_j\sim\mathcal{D}_j}[h^*(z_j)]\\
	&= \mathbb{E}_{\mathcal{D}_1^m}[h^*]
	\end{align*}
	
	Therefore $h^* \in \mathcal{H}^*$. Similarly, for all $h \in \mathcal{H}_1^* \cap \mathcal{H}_2^* \cap \dots \cap \mathcal{H}_m^*$, we have $h \in \mathcal{H}^*$.
	
	On the other hand, if there $\exists \tilde{h} \in \mathcal{H}^*$ but $\tilde{h} \notin \mathcal{H}_1^* \cap \mathcal{H}_2^* \cap \dots \cap \mathcal{H}_m^*$, then at least one of the following condition will hold:
	
	\begin{align*}
	\tilde{h} \notin \mathcal{H}_1^*\ \  or \ \  \tilde{h} \notin \mathcal{H}_2^* \ \ or \ \ \dots \ \ or \ \ \tilde{h} \notin \mathcal{H}_m^*
	\end{align*}
	
	Without loss of generality we assume $\tilde{h} \notin \mathcal{H}_1^*$. Then we have:
	
	\begin{align*}
	\mathbb{E}_{\mathcal{D}_1^m}[\tilde{h}] &= \frac{1}{m} \sum_{j=1}^{m} \mathbb{E}_{z_j\sim\mathcal{D}_j}[\tilde{h}(z_j)]\\
	&> \frac{1}{m} \mathbb{E}_{z_1\sim\mathcal{D}_1}[h^*(z_1)] + \frac{1}{m} \sum_{j=2}^{m} \mathbb{E}_{z_j\sim\mathcal{D}_j}[\tilde{h}(z_j)]\\
	&\ge \frac{1}{m} \sum_{j=1}^{m} \mathbb{E}_{z_j\sim\mathcal{D}_j}[h^*(z_j)]\\
	&= \mathbb{E}_{\mathcal{D}_1^m}[h^*]
	\end{align*}
	
	Therefore $\tilde{h} \notin \mathcal{H}^*$. Which proofs the theorem.
	
\end{proof}

\subsection{Proof of Lemma 2}
\label{sec:lemma2_proof}
\begin{proof}
	We bound $\mathbb{E}_{S_1^m}[\varphi(S)]$ in terms of the Rademacher complexity of $\mathcal{F}$, by introducing a set of 'ghost samples' $T_1^m=\{\tilde{z}_{1,1} \dots \tilde{z}_{j,i} \dots \tilde{z}_{m,n_m}\}$ of $N$ independent samples drawn from $\mathcal{D}_1, \dots \mathcal{D}_m$. We also specifically define $T_j=\{\tilde{z}_{j,1} \dots \tilde{z}_{j,n_j}\}$ as a 'ghost dataset' drawn from $\mathcal{D}_j$.
	
	Let $\sigma = \{\sigma_{1,1} \dots \sigma_{j,i} \dots \sigma_{m,n_m}\}$ be $N$ independent Rademacher random variables. By applying Jensen's inequality and convexity of the supremum function, we have:
	
	\begin{align*}
	\mathbb{E}_{S_1^m}[\varphi(S)]
	&= \mathbb{E}_{S_1^m}\bigg[\sup\limits_{h \in \mathcal{F}}\bigg(\mathbb{E}_{\mathcal{D}_{1}^{m}}[h]-\hat{\mathbb{E}}_{S_1^m}[h]\bigg)\bigg]\\
	&= \mathbb{E}_{S_1^m}\bigg[\sup\limits_{h \in \mathcal{F}}\bigg(\mathbb{E}_{T_1^m}\bigg[\hat{\mathbb{E}}_{T_1^m}[h]-\hat{\mathbb{E}}_{S_1^m}[h] \bigg|S_1^m \bigg]\bigg)\bigg]\\
	\end{align*}
	\begin{align*}
	&= \mathbb{E}_{S_1^m}\bigg[\sup\limits_{h \in \mathcal{F}}\bigg(\mathbb{E}_{T_1^m}\bigg[\frac{1}{m}\sum_{j=1}^{m}\frac{1}{n_j}\sum_{i=1}^{n_j}h(\tilde{z}_{j,i})-\frac{1}{m}\sum_{j=1}^{m}\frac{1}{n_j}\sum_{i=1}^{n_j}h(z_{j,i}) \bigg|S_1^m \bigg]\bigg)\bigg]\\
	&= \mathbb{E}_{S_1^m}\bigg[\sup\limits_{h \in \mathcal{F}}\bigg(\mathbb{E}_{T_1^m}\bigg[\sum_{j=1}^{m}\frac{1}{m n_j}\sum_{i=1}^{n_j} (h(\tilde{z}_{j,i})-h(z_{j,i})) \bigg|S_1^m \bigg]\bigg)\bigg]\\
	&\le \mathbb{E}_{S_1^m}\bigg[\mathbb{E}_{T_1^m}\bigg[\sup\limits_{h \in \mathcal{F}}\bigg(\sum_{j=1}^{m}\frac{1}{m n_j}\sum_{i=1}^{n_j} (h(\tilde{z}_{j,i})-h(z_{j,i})) \bigg) \bigg|S_1^m \bigg]\bigg]\\
	&= \mathbb{E}_{S_1^m,T_1^m}\bigg[\sup\limits_{h \in \mathcal{F}}\bigg(\sum_{j=1}^{m}\frac{1}{m n_j}\sum_{i=1}^{n_j} (h(\tilde{z}_{j,i})-h(z_{j,i})) \bigg) \bigg]\\
	&= \frac{1}{2} \mathbb{E}_{S_1^m,T_1^m}\bigg[\sup\limits_{h \in \mathcal{F}}\bigg( \frac{1}{m n_1}(h(\tilde{z}_{1,1})-h(z_{1,1}))+\frac{1}{m n_1}\sum_{i=2}^{n_1} (h(\tilde{z}_{1,i})-h(z_{1,i}))\\
	&\ \ \ \ +\sum_{j=2}^{m}\frac{1}{m n_j}\sum_{i=1}^{n_j} (h(\tilde{z}_{j,i})-h(z_{j,i})) \bigg) \bigg]+\frac{1}{2} \mathbb{E}_{S_1^m,T_1^m}\bigg[\sup\limits_{h \in \mathcal{F}}\bigg( \frac{1}{m n_1}(h(z_{1,1})-h(\tilde{z}_{1,1}))\\
	&\ \ \ \ +\frac{1}{m n_1}\sum_{i=2}^{n_1} (h(\tilde{z}_{1,i})-h(z_{1,i}))+\sum_{j=2}^{m}\frac{1}{m n_j}\sum_{i=1}^{n_j} (h(\tilde{z}_{j,i})-h(z_{j,i})) \bigg) \bigg]\\
	&= \mathbb{E}_{S_1^m,T_1^m,\sigma_{1,1}}\bigg[\sup\limits_{h \in \mathcal{F}}\bigg( \frac{1}{m n_1}\big(\sigma_{1,1}(h(\tilde{z}_{1,1})-h(z_{1,1}))\big)+\frac{1}{m n_1}\sum_{i=2}^{n_1} (h(\tilde{z}_{1,i})-h(z_{1,i}))\\
	&\ \ \ \ +\sum_{j=2}^{m}\frac{1}{m n_j}\sum_{i=1}^{n_j} (h(\tilde{z}_{j,i})-h(z_{j,i})) \bigg) \bigg]\\
	&\vdots\\
	&= \mathbb{E}_{S_1^m,T_1^m,\sigma}\bigg[\sup\limits_{h \in \mathcal{F}}\bigg(\sum_{j=1}^{m}\frac{1}{m n_j}\sum_{i=1}^{n_j} \sigma_{j,i} (h(\tilde{z}_{j,i})-h(z_{j,i})) \bigg) \bigg]\\
	&\le \frac{1}{m} \sum_{j=1}^{m} \mathbb{E}_{S_1^m,T_1^m,\sigma}\bigg[\sup\limits_{h \in \mathcal{F}}\bigg(\frac{1}{n_j}\sum_{i=1}^{n_j} \sigma_{j,i} (h(\tilde{z}_{j,i})-h(z_{j,i})) \bigg) \bigg]\\
	&= \frac{1}{m} \sum_{j=1}^{m} \mathbb{E}_{S_j,T_j,\sigma}\bigg[\sup\limits_{h \in \mathcal{F}}\bigg(\frac{1}{n_j}\sum_{i=1}^{n_j} \sigma_{j,i} (h(\tilde{z}_{j,i})-h(z_{j,i})) \bigg) \bigg]\\
	&\le \frac{1}{m} \sum_{j=1}^{m} \bigg(\mathbb{E}_{S_j,T_j,\sigma}\bigg[\sup\limits_{h \in \mathcal{F}}\bigg(\frac{1}{n_j}\sum_{i=1}^{n_j} \sigma_{j,i} h(\tilde{z}_{j,i}) \bigg) \bigg] \\
	&\ \ \ \ + \mathbb{E}_{S_j,T_j,\sigma}\bigg[\sup\limits_{h \in \mathcal{F}}\bigg(\frac{1}{n_j}\sum_{i=1}^{n_j} -\sigma_{j,i} h(z_{j,i}) \bigg) \bigg] \bigg)\\
	&= \frac{1}{m} \sum_{j=1}^{m} \bigg(\mathbb{E}_{T_j,\sigma}\bigg[\sup\limits_{h \in \mathcal{F}}\bigg(\frac{1}{n_j}\sum_{i=1}^{n_j} \sigma_{j,i} h(\tilde{z}_{j,i}) \bigg) \bigg] \\
	\end{align*}
	\begin{align*}
	&\ \ \ \ + \mathbb{E}_{S_j,\sigma}\bigg[\sup\limits_{h \in \mathcal{F}}\bigg(\frac{1}{n_j}\sum_{i=1}^{n_j} \sigma_{j,i} h(z_{j,i}) \bigg) \bigg] \bigg)\\
	&= \frac{1}{m} \sum_{j=1}^{m} \bigg(\mathbb{E}_{T_j}\bigg[ \hat{\Re}_{T_j}(\mathcal{F}) \bigg] + \mathbb{E}_{S_j}\bigg[\hat{\Re}_{S_j}(\mathcal{F}) \bigg] \bigg)\\
	&= \frac{2}{m} \sum_{j=1}^{m} \Re_{n_j}(\mathcal{F})
	\end{align*}
	Similarly, we have:
	
	$\mathbb{E}_{S_1^m}[\varphi'(S)]\le \frac{2}{m} \sum_{j=1}^{m} \Re_{n_j}(\mathcal{F})$.
\end{proof}

\subsection{Proof of Theorem 2}
\label{sec:th1_proof}
\begin{proof}
	By the union bound and Lemma~\ref{lemma1}, we have:
	
	\begin{align*}
	&\ \ \ \ \ \mathbb{P}\bigg[\varphi(S)-\mathbb{E}_{S_1^m}[\varphi(S)]\ge\epsilon\ or\ \varphi'(S)-\mathbb{E}_{S_1^m}[\varphi'(S)]\ge\epsilon\bigg]\\
	&\le\mathbb{P}\bigg[\varphi(S)-\mathbb{E}_{S_1^m}[\varphi(S)]\ge\epsilon\bigg]+\mathbb{P}\bigg[\varphi'(S)-\mathbb{E}_{S_1^m}[\varphi'(S)]\ge\epsilon\bigg]\\
	&\le 2e^{\frac{-2 m^2 \epsilon^2}{\sum_{j=1}^{m} \frac{1}{n_j}}}
	\end{align*}
	
	Setting $2e^{\frac{-2 m^2 \epsilon^2}{\sum_{j=1}^{m} \frac{1}{n_j}}}=\delta$, we get $\epsilon= \frac{1}{m} \sqrt{\frac{\log{\frac{2}{\delta}} \sum_{j=1}^{m}\frac{1}{n_j}}{2}}$. Thus:
	
	\begin{align*}
	&\ \ \ \ \  \mathbb{P}\bigg[\max\big(\varphi(S),\varphi'(S)\big)<\max\big(\mathbb{E}_{S_1^m}[\varphi(S)],\mathbb{E}_{S_1^m}[\varphi'(S)]\big)+\epsilon\bigg]\\
	&\ge \mathbb{P}\bigg[\varphi(S)-\mathbb{E}_{S_1^m}[\varphi(S)]<\epsilon\
	and \ \varphi'(S)-\mathbb{E}_{S_1^m}[\varphi'(S)]<\epsilon\bigg]\\
	&= 1-\mathbb{P}\bigg[\varphi(S)-\mathbb{E}_{S_1^m}[\varphi(S)]\ge\epsilon\ 
	or \ \varphi'(S)-\mathbb{E}_{S_1^m}[\varphi'(S)]\ge\epsilon\bigg]\\
	&\ge 1-\delta
	\end{align*}
	
	Notice that:
	\begin{align*}
	(\forall h \in \mathcal{F})\ |\mathbb{E}_{\mathcal{D}_{1}^{m}}[h]-\hat{\mathbb{E}}_{S_1^m}[h]|
	&\le \max\bigg(\sup\limits_{h \in \mathcal{F}}\bigg(\mathbb{E}_{\mathcal{D}_{1}^{m}}[h]-\hat{\mathbb{E}}_{S_1^m}[h]\bigg), \sup\limits_{h \in \mathcal{F}}\bigg(\hat{\mathbb{E}}_{S_1^m}[h]-\mathbb{E}_{\mathcal{D}_{1}^{m}}[h]\bigg)\bigg)\\
	&= \max\big(\varphi(S),\ \varphi'(S)\big)\\
	&\le \max\big(\mathbb{E}_{S_1^m}[\varphi(S)],\ \mathbb{E}_{S_1^m}[\varphi'(S)]\big)+\epsilon
	\end{align*}
	Thus, Lemma~\ref{lemma2} implies:
	
	$(\forall h \in \mathcal{F})\ \  |\mathbb{E}_{\mathcal{D}_{1}^{m}}[h]-\hat{\mathbb{E}}_{S_1^m}[h]|\le \frac{2}{m} \sum_{j=1}^{m} \Re_{n_j}(\mathcal{F})+ \frac{1}{m} \sqrt{\frac{\log{\frac{2}{\delta}} \sum_{j=1}^{m}\frac{1}{n_j}}{2}}$
	
	Therefore:
	
	$d_m(\hat{h},h^*)=\mathbb{E}_{\mathcal{D}_{1}^{m}}[\hat{h}]-\mathbb{E}_{\mathcal{D}_1^m}[h^*]\le \frac{4}{m} \sum_{j=1}^{m} \Re_{n_j}(\mathcal{F})+ \frac{1}{m} \sqrt{2 \log{\frac{2}{\delta}} \sum_{j=1}^{m}\frac{1}{n_j}}$
\end{proof}

\subsection{Proof of Theorem 3}
\label{sec:th2_proof}
\begin{proof}
	By the union bound and Lemma~\ref{lemma1}, we have:
	
	\begin{align*}
	&\mathbb{P}\bigg[\varphi(S)-\mathbb{E}_{S_1^m}[\varphi(S)]\ge\epsilon\ or\ \varphi'(S)-\mathbb{E}_{S_1^m}[\varphi'(S)]\ge\epsilon\ or\ 
	\sum_{j=1}^m p_j \Re_{n_j}(\mathcal{F})-\sum_{j=1}^m p_j \hat{\Re}_{S_j}(\mathcal{F})\ge\epsilon\bigg]\\
	&\le\mathbb{P}\bigg[\varphi(S)-\mathbb{E}_{S_1^m}[\varphi(S)]\ge\epsilon\bigg]+\mathbb{P}\bigg[\varphi'(S)-\mathbb{E}_{S_1^m}[\varphi'(S)]\ge\epsilon\bigg]\\
	&+\mathbb{P}\bigg[\sum_{j=1}^m p_j \Re_{n_j}(\mathcal{F})-\sum_{j=1}^m p_j \hat{\Re}_{S_j}(\mathcal{F})\ge\epsilon\bigg]\\
	&\le 3e^{\frac{-2 m^2 \epsilon^2}{\sum_{j=1}^{m} \frac{1}{n_j}}}
	\end{align*}
	
	Setting $3e^{\frac{-2 m^2 \epsilon^2}{\sum_{j=1}^{m} \frac{1}{n_j}}}=\delta$, we get $\epsilon= \frac{1}{m} \sqrt{\frac{\log{\frac{3}{\delta}} \sum_{j=1}^{m}\frac{1}{n_j}}{2}}$. Thus:
	
	\begin{align*}
	&\ \ \ \ \ \mathbb{P}\bigg[\max\big(\varphi(S),\ \varphi'(S)\big)+2 \sum_{j=1}^{m} p_j \Re_{n_j}(\mathcal{F})<\\
	&\ \ \ \ \ \max\big(\mathbb{E}_{S_1^m}[\varphi(S)],\ \mathbb{E}_{S_1^m}[\varphi'(S)]\big)+2 \sum_{j=1}^m p_j \hat{\Re}_{S_j}(\mathcal{F}) + 3\epsilon\bigg]\\
	&\ge \mathbb{P}\bigg[\varphi(S)-\mathbb{E}_{S_1^m}[\varphi(S)]<\epsilon\ and\ \varphi'(S)-\mathbb{E}_{S_1^m}[\varphi'(S)]<\epsilon\ \\
	&and\ \sum_{j=1}^m p_j \Re_{n_j}(\mathcal{F})-\sum_{j=1}^m p_j \hat{\Re}_{S_j}(\mathcal{F})<\epsilon\bigg]\\
	&= 1-\mathbb{P}\bigg[\varphi(S)-\mathbb{E}_{S_1^m}[\varphi(S)]\ge\epsilon\ or\ \varphi'(S)-\mathbb{E}_{S_1^m}[\varphi'(S)]\ge\epsilon \ \\
	&or\ \sum_{j=1}^m p_j \Re_{n_j}(\mathcal{F})-\sum_{j=1}^m p_j \hat{\Re}_{S_j}(\mathcal{F})\ge\epsilon\bigg]\\
	&\ge 1-\delta
	\end{align*}
	
	Notice that:
	
	\begin{align*}
	(\forall h \in \mathcal{F})\ |\mathbb{E}_{\mathcal{D}_{1}^{m}}[h]-\hat{\mathbb{E}}_{S_1^m}[h]|&\le \max\bigg(\sup\limits_{h \in \mathcal{F}}\bigg(\mathbb{E}_{\mathcal{D}_{1}^{m}}[h]-\hat{\mathbb{E}}_{S_1^m}[h]\bigg),\sup\limits_{h \in \mathcal{F}}\bigg(\hat{\mathbb{E}}_{S_1^m}[h]-\mathbb{E}_{\mathcal{D}_{1}^{m}}[h]\bigg)\bigg)\\
	&= \max\big(\varphi(S),\ \varphi'(S)\big)\\
	&\le \max\big(\mathbb{E}_{S_1^m}[\varphi(S)],\ \mathbb{E}_{S_1^m}[\varphi'(S)]\big)+\epsilon
	\end{align*}
	
	Thus, Lemma~\ref{lemma1} and Lemma~\ref{lemma2} implies:
	
	$(\forall h \in \mathcal{F})\ \  |\mathbb{E}_{\mathcal{D}_{1}^{m}}[h]-\hat{\mathbb{E}}_{S_1^m}[h]|\le \frac{2}{m} \sum_{j=1}^{m} \hat{\Re}_{S_j}(\mathcal{F}) + \frac{3}{m} \sqrt{\frac{\log{\frac{3}{\delta}} \sum_{j=1}^{m}\frac{1}{n_j}}{2}}$
	
	Therefore:
	
	$d_m(\hat{h},h^*)=\mathbb{E}_{\mathcal{D}_{1}^{m}}[\hat{h}]-\mathbb{E}_{\mathcal{D}_1^m}[h^*]\le \frac{4}{m} \sum_{j=1}^{m} p_j \hat{\Re}_{S_j}(\mathcal{F}) + \frac{1}{m} \sqrt{18 \log{\frac{3}{\delta}} \sum_{j=1}^{m}\frac{1}{n_j}}$
\end{proof}

\subsection{Proof of Theorem 4}
\label{sec:th3_proof}
\begin{proof}
	By Theorem~\ref{th_expected_R} and the Cauchy-Schwarz inequality, we have:
	
	\begin{align}
	\label{d_m_cauchy}
	d_m(\hat{h},h^*)&= \frac{4}{m} \sum_{j=1}^{m} \Re_{n_j}(\mathcal{F}) + \frac{1}{m} \sqrt{2 \log{\frac{2}{\delta}} \sum_{j=1}^{m}\frac{1}{n_j}} \nonumber \\
	&\le \frac{1}{m} \bigg( 4\sum_{j=1}^{m} \frac{a_j}{\sqrt{n_j}}  +\sqrt{2\log{\frac{2}{\delta}} \sum_{j=1}^{m}\frac{1}{n_j}} \bigg) \nonumber \\
	&= \frac{1}{m} \bigg( \sum_{j=1}^{m} \sqrt{\frac{16 a_j^2}{n_j}}  +\sqrt{2 \log{\frac{2}{\delta}} \sum_{j=1}^{m}\frac{1}{n_j}} \bigg) \nonumber \\
	&\le \frac{\sqrt{m+1}}{m} \sqrt{\sum_{j=1}^{m} \frac{16 a_j^2+2\log{\frac{2}{\delta}}}{n_j}}
	\end{align}
	
	Now the question is how to set $n_j$ in order to minimize the bound obtained in \eqref{d_m_cauchy}.
	
	Define $\gamma_j= 16 a_j^2+2\log{\frac{2}{\delta}}$. We can define the following optimization problem:
	
	$\min{\sum_{j=1}^{m} \frac{\gamma_j}{n_j}}$, s.t. $\sum_{j=1}^{m} c_j n_j \le C,\ n_j> 0.$
	
	The dual problem is:
	
	$\max{2\sum_{j=1}^{m} \sqrt{\gamma_j (c_j \lambda_0 - \lambda_j) } - C \lambda_0}$, s.t. $\lambda_j\ge0, \lambda_0>\max\limits_j{\frac{\lambda_j}{c_j}}.$
	
	Strong duality holds due to the linearity of the constraints in the primal problem and Slater's condition (\cite{boyd04convex}).
	
	It is easy to see that due to complementary slackness, all $\lambda_j=0$, and $\lambda_0$ can be solved by taking the derivative of the dual objective equal to zero.
	
	Therefore the dual problem reaches maximum when $\nu=\frac{(\sum_{j=1}^{m} \sqrt{\gamma_j c_j})^2}{C^2}, \lambda_j=0, j=1,\dots,m$.
	
	Thus, if
	
	\begin{equation}
	\label{n_j_split}
	n_j=C \frac{\sqrt{\gamma_j}}{\sqrt{c_j} \sum_{k=1}^{m} \sqrt{\gamma_k c_k}}
	\end{equation}
	
	the primal problem reaches its minimum. By replacing $n_j$ with \eqref{n_j_split} into either Theorem~\ref{th_expected_R} or \eqref{d_m_cauchy}, we have:
	
	$d_m(\hat{h},h^*)\le \frac{\sqrt{\sum_{j=1}^m \sqrt{16 a_j^2 c_j + 2 c_j \log{\frac{2}{\delta}}}}}{m\sqrt{C}}$ 
	
	$\bigg[ \sum_{j=1}^m \frac{4 a_j \sqrt[4]{c_j}}{\sqrt[4]{16 a_j^2 + 2 \log{\frac{2}{\delta}}}} + \sqrt{2 \log{\frac{2}{\delta}} \sum_{j=1}^m \frac{\sqrt{c_j}}{\sqrt{16 a_j^2 + 2 \log{\frac{2}{\delta}}}}} \bigg]$
	
	$\le \frac{\sqrt{m+1}}{m\sqrt{C}}\sum_{j=1}^m \sqrt{16 a_j^2 c_j +2 c_j \log{\frac{2}{\delta}}}$.
	
\end{proof}

\subsection{Proofs of Example Statements}
\subsubsection{Two-Layer Neural Networks}
\label{sec:neural_networks}
\begin{proof}
	Recall that the empirical Gaussian complexity of $\mathcal{G}$ with respect to the dataset $S_j$ of $n_j$ samples is defined as:
	
	\begin{equation*}
	\hat{G}_{S_j}(\mathcal{G})=\mathbb{E}_{g}\bigg[\sup\limits_{h \in \mathcal{G}} \bigg(\frac{1}{n_j}\sum_{i=1}^{n_j}g_i z_{j,i}\bigg)\bigg]
	\end{equation*}
	
	where $g=\{g_1, \dots g_{n_j}\}$ are $n_j$ independent Gaussian $N(0,1)$ random variables. The Gaussian complexity of $\mathcal{G}$ for $n_j$ samples is defined as:
	
	\begin{equation*}
	G_{n_j}(\mathcal{G})=\mathbb{E}_{S_j\sim\mathcal{D}_j^{n_j}}[\hat{G}_{S_j}(\mathcal{G})]
	\end{equation*}
	
	From Theorem 18 in \cite{bartlett02}, the empirical Gaussian complexity of a two-layer neural network can be bounded by:
	
	\begin{align*}
	\hat{G}_{S_j}(\mathcal{G}) &\le \frac{b}{n_j} (\log{l})^{\frac{1}{2}} \max_{k,k'} \big( \sum_{i=1}^{n_j} (x_{j,i}^{(k)}-x_{j,i}^{(k')})^2 \big)^{\frac{1}{2}} \\
	&\le \frac{2b}{n_j} (\log{l})^{\frac{1}{2}} \sqrt{n_j} X_{\infty,j} \\
	&= \frac{2b (\log{l})^{\frac{1}{2}} X_{\infty,j}}{\sqrt{n_j}} 
	\end{align*}
	
	where $x_{j,i}=\{x_{j,i}^{(1)}, \dots x_{j,i}^{(l)}\}$ and $b>0$ is an absolute constant.
	
	It is obvious that
	
	\begin{equation*}
	G_{n_j}(\mathcal{G})=\mathbb{E}_{S_j\sim\mathcal{D}_j^{n_j}}[\hat{G}_{S_j}(\mathcal{G})]\le \frac{2b (\log{l})^{\frac{1}{2}} X_{\infty,j}}{\sqrt{n_j}} 
	\end{equation*}
	
	From Lemma 4 in \cite{bartlett02}, we have for an absolute constant $b'>0$:
	
	\begin{equation}
	\label{lemma_bartlett}
	\Re_{n_j}(\mathcal{G}) \le b' G_{n_j}(\mathcal{G})
	\end{equation}
	
	By \eqref{LT_exp} and \eqref{lemma_bartlett} we have:
	
	\begin{equation*}
	\Re_{n_j}(\mathcal{F})\le\frac{B X_{\infty,j} \sqrt{\log{l}} }{\sqrt{n_j}}
	\end{equation*}
	
	with $B>0$ being an absolute constant.
\end{proof}

\subsubsection{Kernel Methods}
\label{sec:kernel_methods}
\begin{proof}
	From Lemma 22 in \cite{bartlett02} and Jensen's inequality, we have:
	
	\begin{align}
	\label{kernel_methods}
	\Re_{n_j}(\mathcal{G})&=\mathbb{E}_{S_j\sim\mathcal{D}_j^{n_j}}[\hat{\Re}_{S_j}(\mathcal{G})]\nonumber\\
	&\le \mathbb{E}_{S_j\sim\mathcal{D}_j^{n_j}}\Bigg[\frac{2 B_j}{n_j} \sqrt{\sum_{i=1}^{n_j} k(x_{j,i}, x_{j,i})}\Bigg]\nonumber\\
	&\le \frac{2 B_j \sqrt{\mathbb{E}_{x_j}[k(x_j, x_j)]}}{\sqrt{n_j}}
	\end{align}
	
	By \eqref{LT_exp} and \eqref{kernel_methods} we have:
	
	\begin{equation*}
	\Re_{n_j}(\mathcal{F})\le\frac{2 B_j \sqrt{\mathbb{E}_{x_j}[k(x_j, x_j)]}}{\sqrt{n_j}}
	\end{equation*}
\end{proof}

\subsection{Experiment Details}
\label{sec:experiment_details}
Due to the independence between $\zeta_{j}$ and $\epsilon$, as well as the fact that $\mathbb{E}[{\zeta_j^{(k)}}^2]=var(\zeta_j^{(k)})=\frac{4 X_{2,j}^2}{3l^3}$:
\begin{align*}
d_m (\hat{h}, h^*)\nonumber &= \mathbb{E}_{\mathcal{D}_{1}^{m}}[\hat{h}]-\mathbb{E}_{\mathcal{D}_{1}^{m}}[h^*]\\
&= \sum_{j=1}^m \frac{1}{2m} (\mathbb{E}[(\hat{w}^T x_j-y_j)^2]-\mathbb{E}[(w^{*T} x_j-y_j)^2])\\
&= \sum_{j=1}^m \frac{1}{2m} (\mathbb{E}[(\hat{w}^Tx_j-w^{*T}x_j-\epsilon)^2]-\mathbb{E}[\epsilon^2])\\
&= \sum_{j=1}^m \frac{1}{2m} (\mathbb{E}[(\hat{w}^Tx_j-w^{*T}x_j)^2]-2(\hat{w}-w^*)^TA_j\mathbb{E}[\zeta_j\epsilon])\\
&= \sum_{j=1}^m \frac{1}{2m} \sum_{k,k'} \mathbb{E}[((\hat{w}-w^*)^T A_j)^{(k)} \zeta_j^{(k)} ((\hat{w}-w^*)^T A_j)^{(k')} \zeta_j^{(k')}]\\
&= \sum_{j=1}^m \frac{1}{2m} \sum_{k=1}^l \mathbb{E}[{((\hat{w}-w^*)^T A_j)^{(k)}}^2 {\zeta_j^{(k)}}^2]\\
&= \frac{2}{3ml^3} \sum_{j=1}^m X_{2,j}^2 \Vert (\hat{w}-w^*)^T A_j \Vert_2^2
\end{align*}

\end{document}